\documentclass{article}
\usepackage{ijcai22}

\usepackage{xspace}
\usepackage{latexsym,amssymb,amsfonts}
\usepackage[defblank]{paralist}
\usepackage[textsize=footnotesize,textwidth=2.5cm]{todonotes}
\usepackage{tikz}
\usepackage{tikz-cd}
\usepackage{array,multirow}
\usepackage{comment}
\usepackage{enumerate}
\usepackage[linesnumberedhidden,ruled,vlined]{algorithm2e}
\usepackage{url}

\DeclareMathAlphabet{\mathcalligra}{T1}{calligra}{m}{n}

\usepackage{times}
\usepackage{soul}
\usepackage{url}
\usepackage{hyperref}
\usepackage[utf8]{inputenc}
\usepackage[small]{caption}
\usepackage{graphicx}
\usepackage{amsmath}
\usepackage{amsthm}
\usepackage{booktabs}
\usepackage{algorithmic}
\newtheorem{example}{Example}
\newtheorem{theorem}{Theorem}
\newtheorem{definition}{Definition}
\newtheorem{corollary}{Corollary}
\newtheorem{proposition}{Proposition}
\newtheorem{lemma}{Lemma}
\newtheorem{remark}{Remark}

\title{Relational Action Bases:\\ Formalization, Effective Safety Verification, and Invariants \\ (Extended Version)}

\author{
Silvio Ghilardi$^1$
\and
Alessandro Gianola$^2$\and
Marco Montali$^2$\And
Andrey Rivkin$^2$
\affiliations
$^1$Universit\`a degli Studi di Milano\\
$^2$Free University of Bozen-Bolzano
\emails
silvio.ghilardi@unimi.it,
\{gianola,montali,rivkin\}@inf.unibz.it
}

\newcommand{\ua}{\ensuremath{\underline a}}

\newcommand{\ud}{\ensuremath{\underline d}}
\newcommand{\ue}{\ensuremath{\underline e}}

\newcommand{\ui}{\ensuremath{\underline i}}

\newcommand{\ux}{\ensuremath{\underline x}}

\newcommand{\uy}{\ensuremath{\underline y}}
\newcommand{\uz}{\ensuremath{\underline z}}

\newcommand{\cSi}{\ensuremath \mathcal S}

\newcommand{\DB}{\mathsf{SDB}}
\newcommand{\artDB}{\mathsf{SDB}}

\newcommand{\LRA}{\ensuremath{\mathcal{LRA}}\xspace}
\newcommand{\LIA}{\ensuremath{\mathcal{LIA}}\xspace}

\newcommand{\cA}{\ensuremath \mathcal A}
\newcommand{\cC}{\ensuremath \mathcal C}
\newcommand{\cB}{\ensuremath \mathcal B}
\newcommand{\cM}{\ensuremath \mathcal M}

\newcommand{\cS}{\ensuremath \mathcal S}

\newcommand{\safe}{\texttt{SAFE}\xspace}
\newcommand{\unsafe}{\texttt{UNSAFE}\xspace}
\newcommand{\mcmt}{\textsc{mcmt}\xspace}

\renewcommand{\int}{\ensuremath {\mathcal I}}

\newcommand{\EUF}{\ensuremath{\mathcal{EUF}}}

\newcommand{\ras}{RAS\xspace}
\newcommand{\uras}{RAB\xspace}

\newcommand{\univtrans}{\mathcalligra{t}\mathcalligra{r}\xspace}

\newcommand{\type}[1]{\ensuremath{\mathsf{#1}}\xspace}
\newcommand{\funct}[1]{\ensuremath{\mathit{#1}}\xspace}

\newcommand{\constant}[1]{\texttt{#1}}
\newcommand{\Sorts}[1]{#1_{\mathit{srt}}}

\newcommand{\functs}[1]{#1_{\mathit{fun}}}
\newcommand{\vals}[1]{#1_{\mathit{val}}}
\newcommand{\ids}[1]{#1_{\mathit{ids}}}
\newcommand{\ext}[1]{#1_{\mathit{ext}}}

\newcommand{\nullv}{\type{undef}}

\definecolor{mgreen}{rgb}{0.0, 0.5, 0.0}

\definecolor{deepblue}{HTML}{0C3B80}
\definecolor{deepgreen}{HTML}{2EA601}
\definecolor{lightOrange}{HTML}{FFA03C}
\definecolor{darkOrange}{HTML}{F1800A}
\definecolor{lightBlue}{HTML}{0174CD}
\definecolor{greenF}{HTML}{2CBB5C}
\definecolor{cyan}{HTML}{86A6D5}

\tikzstyle{sortnode} = [
  rectangle,
  rounded corners=5pt,
  minimum width=18mm,
  minimum height=6mm,
  very thick,
]

\tikzstyle{functnode} = [
  rectangle,
  minimum width=1cm,
  minimum height=5mm,
]

\tikzstyle{idnode} = [
  sortnode,
  draw=deepblue,
  fill=cyan!50,
  text=deepblue,  
]

\tikzstyle{valnode} = [
  sortnode,
  densely dashed,
  draw=violet,
  fill=magenta!50,
  text=violet,
]

\tikzstyle{f} = [
  thick,
]

\tikzstyle{fd} = [
  f,
  -angle 90,
]

\tikzstyle{relation}=[rectangle split, rectangle split parts=#1, rectangle split part align=base, draw, anchor=center, align=center, text height=3mm, font=\bfseries, ultra thick, text centered]

\newcommand{\smtlambda}[2]{
  \lambda #1.\left(#2\right)
}
\newcommand{\smtif}[4]{
  \begin{array}[#1]{@{}l@{}}
  \mathsf{if~}#2\mathsf{~then~}#3
  \\\mathsf{else~}#4
  \end{array}
}

\newcommand{\smtifinline}[3]{
  \mathsf{if~}#1\mathsf{~then~}#2\mathsf{~else~}#3
}

\newcommand{\typedvar}[2]{#1{:}#2}

\newcommand{\artvar}[1]{\mathit{#1}}

\newcommand{\vsig}{\Sigma_{\mathit{v}}}
\newcommand{\vras}{\S_{\mathit{va}}}
\newcommand{\nin}{\type{NIN}}

\newcommand{\stringval}{\type{String}}
\newcommand{\intval}{\type{Int}}

\newcommand{\pstatev}{\artvar{pState}}

\newcommand{\visastat}{\artvar{visaStat}}
\newcommand{\tonotify}{\artvar{toNotify}}

\newcommand{\enab}{\constant{enabled}}
\newcommand{\evalu}{\constant{evaluated}}

\newcommand{\notified}{\constant{notified}}

\newcommand{\appidx}{\type{appIndex}}

\newcommand{\appuser}{\funct{applicant}}

\newcommand{\appscore}{\funct{appScore}}
\newcommand{\appres}{\funct{appResult}}

\newcommand{\appvisa}{\funct{vType}}

\tikzstyle{sortnode} = [
  rectangle,
  rounded corners=5pt,
  minimum width=18mm,
  minimum height=6mm,
  very thick,
]

\tikzstyle{functnode} = [
  rectangle,
  minimum width=1cm,
  minimum height=5mm,
]

\tikzstyle{idnode} = [
  sortnode,
  draw=deepblue,
  fill=cyan!50,
  text=deepblue,  
]

\tikzstyle{artnode} = [
  sortnode,
  draw=deepblue,
  fill=yellow!50,
  text=deepblue,  
]

\tikzstyle{valnode} = [
  sortnode,
  densely dashed,
  draw=violet,
  fill=magenta!50,
  text=violet,
]

\tikzstyle{f} = [
  thick,
]

\tikzstyle{fd} = [
  f,
  -angle 90,
]

\tikzstyle{relation}=[rectangle split, rectangle split parts=#1, rectangle split part align=base, draw, anchor=center, align=center, text height=3mm, font=\bfseries, text centered]

\tikzstyle{outflow} = [sequence,->,densely dotted]
\tikzstyle{smalltask} = [rectangle,draw,rounded corners=5pt,minimum height=2.5em,minimum width=3em]
\tikzstyle{block} = [task,densely dotted]
\tikzstyle{smallblock} = [smalltask,densely dotted]
\tikzstyle{legend} = [font=\footnotesize]

\tikzstyle{state} = [
  rectangle,
  draw,
  rounded corners=10pt,
  minimum width=15mm,
  minimum height=7mm]

\definecolor{deepblue}{HTML}{0C3B80}
\definecolor{deepgreen}{HTML}{2EA601}
\definecolor{lightOrange}{HTML}{FFA03C}
\definecolor{darkOrange}{HTML}{F1800A}
\definecolor{lightBlue}{HTML}{0174CD}
\definecolor{greenF}{HTML}{2CBB5C}
\definecolor{cyan}{HTML}{86A6D5}

\tikzstyle{task} = [
  minimum height=1cm,
  minimum width = 1.4cm,
  rectangle,
  fill=white,
  thick,
  draw,
  rounded corners,
  align=center,
  every task, 
  font=\footnotesize
]

\tikzstyle{sequence} = [
  ->,
  >=triangle 45,
  every sequence,
  thick
]

\tikzstyle{guard} = [
  rectangle,
  draw=none,
  fill=white,
  inner sep=.5mm,
  minimum height=.8mm,
  minimum width=.8mm
]

\tikzstyle{lbl} = [text width=4cm]

\makeatletter
\pgfarrowsdeclare{center*}{center*}
{
  \pgfarrowsleftextend{+-.5\pgflinewidth}
  \pgfutil@tempdima=0.4pt%
  \advance\pgfutil@tempdima by.2\pgflinewidth%
  \pgfarrowsrightextend{4.5\pgfutil@tempdima}
}
{
  \pgfutil@tempdima=0.4pt%
  \advance\pgfutil@tempdima by.2\pgflinewidth%
  \pgfsetdash{}{+0pt}
  \pgfpathcircle{\pgfqpoint{4.5\pgfutil@tempdima}{0bp}}{4.5\pgfutil@tempdima}
  \pgfusepathqfillstroke
}

\pgfarrowsdeclare{centero}{centero}
{
  \pgfarrowsleftextend{+-.5\pgflinewidth}
  \pgfutil@tempdima=0.4pt%
  \advance\pgfutil@tempdima by.2\pgflinewidth%
  \pgfarrowsrightextend{4.5\pgfutil@tempdima}
}
{
  \pgfutil@tempdima=0.4pt%
  \advance\pgfutil@tempdima by.2\pgflinewidth%
  \pgfsetdash{}{+0pt}
  \pgfpathcircle{\pgfqpoint{4.5\pgfutil@tempdima}{0bp}}{4.5\pgfutil@tempdima}
  \pgfusepathqstroke
}
\makeatother

\newcommand{\B}{\ensuremath{\mathcal{B}}}

\renewcommand{\S}{\ensuremath{\mathcal{S}}}
\newcommand{\T}{\ensuremath{T}}

\newcommand{\ApproxB}{\tilde{B}}
\newcommand{\ApproxPre}{\tilde{\mathit{Pre}}}
\newcommand{\InstPre}{\mathit{InstPre}}

\newcommand{\tup}[1]{\langle #1\rangle}            

\begin{document}

\maketitle

\begin{abstract}
Modeling and verification of dynamic systems operating over a relational representation of states are increasingly investigated problems in AI, Business Process Management, and Database Theory. To make these systems amenable to verification, the amount of information stored in each relational state needs to be bounded, or restrictions are imposed on the preconditions and effects of actions. We introduce the general framework of relational action bases (RABs), which generalizes existing models by lifting both these restrictions: unbounded relational states can be evolved through actions that can quantify both existentially and universally over the data, and that can exploit numerical datatypes with arithmetic predicates. We then study parameterized safety of RABs via (approximated) SMT-based backward search, singling out essential meta-properties of the resulting procedure, and showing how it can be realized by an off-the-shelf combination of existing verification modules of the state-of-the-art MCMT model checker. We demonstrate the effectiveness of this approach on a benchmark of data-aware business processes. Finally, we show how universal invariants can be exploited to make this procedure fully correct.

\end{abstract}


\section{Introduction}

Reasoning about actions and processes has lately witnessed an important shift in the representation of states and the way actions query and progress them. 
Relational representations and query/update languages are often used in these systems as a convenient syntactic sugar to compactly represent and evolve propositional states; this is, for example, what happens in STRIPS for planning.  
More recently, instead, several works have put emphasis on states captured by full-fledged relational structures, equipping dynamic systems with actions that can create and destroy objects and relations. This \emph{data-awareness} is essential to capture relevant systems in AI \cite{BarD15} and business process management \cite{Vian09,CaDM13}. Notable examples are: 
\begin{inparaenum}[\itshape (i)]
\item Situation Calculus- \cite{DeLP16} and knowledge-based \cite{HaririCMGMF13} action theories, 
\item relational MDPs \cite{Yang-ML2022},
\item database-enriched business processes \cite{BCDDM13,DeLV16,BPM19}, and \item Petri nets with identifiers \cite{PolyvyanyyWOB19,IS22} and with advanced interaction mechanisms \cite{Fahland19}.
\end{inparaenum}
The subtle interplay of the data and process dimensions calls for suitable techniques for verifying at design-time whether such \emph{relational dynamic systems} behave as  expected \cite{BarD15,CaDM13}.
However, verification is extremely difficult in this setting, as a relational dynamic system induces, in general, an infinite state-space. This poses a twofold challenge for verification procedures: \emph{foundationally}, the identification of interesting classes of systems for which suitable verification procedures enjoy key (meta-)properties (such as soundness, completeness, correctness, and termination); \emph{practically}, the development of corresponding effective verification tools. 

To attack these two challenges, two main lines of research have been pursued. In the first, a bound is imposed on the number of objects that can be stored in a single state \cite{HaririCD14}. This makes verification of first-order $\mu$-calculus \cite{CDMP18} and for a fragment of first-order LTL \cite{CalvaneseGMP22} reducidble to conventional finite-state model checking, for systems with a fixed initial state and objects only equipped with equality comparison. The reductions are not helpful to understand how to practically approach verification from the algorithmic point of view, and only preliminary results exist in connection to practical verifiers \cite{Yang-ML2022}. 
In the second line of research, unbounded relational dynamic systems are studied, where unboundedness refers to two distinct aspects. On the one hand, the working memory of the system can contain unboundedly many objects per state; on the other hand, verification is studied parametrically to read-only data, to ensure that desired properties hold no matter how such unmodifiable data are fixed. This setting, surveyed in \cite{Vian09}, is more delicate than the one of state-bounded systems: the identification of verifiable fragments calls for a very careful analysis of which modelling features are supported to specify actions. At the same time, perhaps surprisingly, practical verifiers exist, based on explicit-state model checking and vector addition systems \cite{verifas}, or on symbolic model checking via SMT \cite{BPM19} implemented in a dedicated module of the MCMT model checker \cite{mcmt}. The latter approach covers the most general modeling framework in this spectrum, called RAS \cite{MSCS20}, which extends the framework in \cite{DeLV16} while restricting verification to (parameterized) \emph{safety} properties, to check that for every possible instantiations of the read-only data, the RAS never reaches undesired states. In MCMT, this is realized through a backward reachability procedure that iteratively computes preimages of unsafe states, employing dedicated quantifier elimination techniques for data variables \cite{MSCS20,JAR21}.

A RAS comes with 
\begin{inparaenum}[\it (i)]
\item a read-only database with primary and foreign keys, 
\item a working memory consisting of read-write unbounded relations, and \item actions that captures several forms of constrained updates, whose preconditions query the working and read-only memory with existential queries, and whose effects consist in update formulae that can express addition, deletion, or bulk update of tuples in the read-write relations.
\end{inparaenum}
However, two important features are missing: 
\begin{inparaenum}[\it (i)]
\item numerical datatypes with arithmetic conditions (limited forms are only studied in \cite{DeLV16}, without further implementation in \cite{verifas}), and
\item queries with universal quantification. 
\end{inparaenum}

The first contribution of this work is to generalize RASs in the new framework of \emph{relational action bases (RABs)}, which lifts both these restrictions: RABs can evolve unbounded relational states can through an actions that quantify both existentially and universally over the data, and that can exploit different numerical datatypes with arithmetic predicates (including linear arithmetics for integers and reals). As substantiated in Section~\ref{sec:usefulness}, this is essential to capture fundamental modelling features arising in different application domains, and that where out of reach so far.

On top of this framework, we provide a threefold technical contribution. First and foremost, we study how to lift the SMT-based backward search employed by MCMT tho the richer setting of RABs.  This requires to introduce two (over-)approximation steps in the symbolic computation of unsafe states, which could cause the procedure to detect unsafety spuriously. This calls for a detailed investigation of the meta-properties of the procedure. Since the setting of RABs is highly undecidable, we concentrate on soundness, completeness, and correctness of the procedure; decidable cases for which termination is also guaranteed are inherited from the RAS fragments identified in \cite{MSCS20}. Importantly, this investigation has a direct, practical impact: it shows that existing separate modules of MCMT model checker, respectively implementing \emph{cover} computation to handle arithmetics \cite{JAR21,JAR22}, and dynamic forms of \emph{instantiation} \cite{jsat,shahom} to handle universally quantified variables, can be gracefully combined with a third MCMT module dedicated to quantifier elimination for data variables. This in turn witnesses that such modules can be effectively combined in MCMT to verify RABs off-the-shelf, without destroying the good meta-properties of the combined procedure.

As a second technical contribution, we then put MCMT at work, showing that it terminates returning the correct answer with a very effective performance, on a non-trivial set of 27 different RABs constructed from a benchmark of data-aware business processes \cite{verifas}.

The last technical contribution is on the injection on invariants in the safety analysis of RABs. Taking inspiration from \cite{bjorner,bjorner17}, we show, for the first time in the context of backward reachability for this class of dynamic systems, that spurious results for unsafety are incompatible with the existence of universal invariants. 

\section{The Need of Arithmetics and Universals}\label{sec:usefulness}
We briefly motivate here why arithmetics and universally quantified variables are essential when modeling relational dynamic systems in different application domains. As the need for numerical variables and arithmetics is well-known (see, e.g., \cite{GereviniSS08,Belardinelli14,DeLV16,FelliMW22}), we concentrate on universal quantification. 

\smallskip
\noindent
\textbf{Injection of fresh objects.} Relational dynamic systems need to create new, fresh objects during the execution. Classical examples are the creation of a new primary key within a database \cite{BCDDM13,BeLP14}, or of a new identifier in a (high-level) Petri net \cite{PolyvyanyyWOB19,IS22}. Without universal quantification, freshness cannot be guaranteed: it is possible to nondeterministically pick an object, but not to enforce that it is not contained in the working memory. 

\smallskip
\noindent
\textbf{Conditional updates and constraint checking.}
Universally quantified variables are essential to express rich conditional updates operating at once over the entire extension of a relation \cite{BCDDM13,BeLP14}. Conditional updates can also be used to enforce constraints on the working memory of the system, using the following modelling pattern. Considering the forms of actions supported by RABs, the constraint of interest can be formulated as a universally quantified sentence. The system alternates an action mode and a check mode. In the check mode, the system verifies whether the constraint holds. If so, the system goes back to the action mode. If not (i.e., the negation of the constraint holds), the system enters into an error state. Universally quantified sentences can express, notably, key and disjointness constraints. 

\smallskip
\noindent
\textbf{Universal synchronization semantics.}
Dynamic systems where multiple entities progress concurrently typically call for synchronization mechanisms. Two key examples are (parameterized) multiagent systems where \emph{all} agents synchronously perform a joint action \cite{KouvarosAIJ16,PRIMA2020}, and the equality synchronization semantic of proclets \cite{Fahland19}, where a transition can be performed by a parent object only if \emph{all} its child object are in a certain state (e.g., all items of the same order have been validated). Thanks to universally quantified variables, RABs support the such synchronization mechanisms. This is especially remarkable considering that, in particular, verification of the proclet model was out of reach so far \cite{Fahland19,IS22}).

\section{Preliminaries}
\label{sec:prelim}

We adopt the usual first-order syntactic notions of signature, term,
atom, (ground) formula,  and the like. In general, signatures 
 are multi-sorted, and every sort comes with equality. For simplicity, most definitions 
 will be given for single-sorted languages; the adaptation to  multi-sorted languages is straightforward.
We compactly represent a tuple $\tup{x_1,\ldots,x_n}$ of variables as $\ux$. Notation $t(\ux)$ (resp.,~$\phi(\ux)$) means that term $t$ (resp.,~formula $\phi$) has free variables included in the tuple $\ux$. 
 We always assume that terms and formulae are well-typed. 
A formula is said to be \emph{universal} (resp., \emph{existential}) if it has the form $\forall \ux (\phi(\ux))$ (resp., $\exists \ux (\phi(\ux))$), where $\phi$ is a quantifier-free formula. A \emph{sentence} is a formula without free variables.

For semantics, we use the standard notions of  $\Sigma$-structure $\cM$ and of truth of a formula in a $\Sigma$-structure under a free variables assignment. 
A \emph{$\Sigma$-theory} $T$ is a set of $\Sigma$-sentences; a \emph{model}  of $T$ is a $\Sigma$-structure $\cM$ where all sentences in $T$ are true.
	We use 
	$T\models \phi$, indicating that $\phi$ {\it is a logical consequence of} $T$, to express that $\phi$ is true in all models of $T$ for every assignment to the variables occurring free in $\phi$.
We say that $\phi$ is \emph{$T$-satisfiable} iff there exists a model $\cM$ of $T$ and an assignment to the free variables of $\phi$ that makes $\phi$ true in $\cM$: if $\phi$ is quantifier-free, the problem of establishing for $\phi$ the existence of such a model and such an assignment is called \emph{SMT problem} for $T$. Examples of theories from the SMT literature are $\EUF$, the theory of equality with uninterpreted symbols, and $\LIA$/$\LRA$, the theory of linear integer/real arithmetics (see, e.g., \cite{smt-lib} for details).
	 A $T$-cover of a formula $\exists \ux \phi(\ux,\uy)$ is the strongest quantifier-free formula $\psi(\uy)$ that is implied by $\exists \ux \phi(\ux,\uy)$ modulo $T$, i.e., it is implied by $\exists \ux \phi(\ux,\uy)$ and implies all the other implied formulae $\phi'(\uz,\uy)$ modulo $T$. A theory $T$ has \emph{quantifier-free uniform interpolation} iff all formulae $\exists \ux \phi(\ux,\uy)$ have a $T$-cover, and an effective procedure for computing them is available. Computing $T$-covers is strictly related to the problem of eliminating quantifiers in suitable theory extensions of $T$ \cite{JAR21}. 

	

 

We define now 
case-defined functions, abbreviating more complicated (still first-order)
expressions.  Fix 
$\Sigma$-theory $T$; a
\emph{$T$-partition} is a finite set $\kappa_1(\ux), \dots, \kappa_n(\ux)$ of quantifier-free formulae
 s. t. $T\models \forall \ux \bigvee_{i=1}^n \kappa_i(\ux)$ and
$T\models \bigwedge_{i\not=j}\forall \ux \neg (\kappa_i(\ux)\wedge
\kappa_j(\ux))$.  Given such a $T$-partition
$\kappa_1(\ux), \dots, \kappa_n(\ux)$ together with $\Sigma$-terms
$t_1(\ux), \dots, t_n(\ux)$ (all of the same target sort), a
\emph{case-definable extension} is the $\Sigma'$-theory $T'$ where
$\Sigma'=\Sigma\cup\{F\}$, with $F$ a ``fresh'' function symbol (i.e.,
$F\not\in\Sigma$) 
, and
$T'=T \cup\bigcup_{i=1}^n \{\forall\ux\; (\kappa_i(\ux) \to F(\ux) =
t_i(\ux))\}$.
%
%
Intuitively, $F$ represents a case-defined function, representable
using nested if-then-else expressions as:
$
F(\ux) ~:=~ \mathtt{case~of}~
\{\kappa_1(\ux):t_1;\cdots;\kappa_n(\ux):t_n\}.
$
We identify $T$ with any of its case-definable
extensions $T'$.  In fact, 
given a $\Sigma'$-formula
$\phi'$, one can easily find a $\Sigma$-formula $\phi$ that is equivalent to $\phi'$ in all models
of $T'$. 

Again for compactness, we also use $\lambda$-abstractions. 
We always abbreviate formulae of the form $\forall y.~b(y)=F(y,\uz)$
(where, typically, $F$ is a symbol introduced in a case-defined extension) into $b = \lambda y. F(y,\uz)$. Hence, also these $lambda$-abstractions can be converted back into plain first-order formulae.


\section{Relational Action Bases}
\label{sec:model}
We are now ready to introduce the general formal model of \emph{relational action bases} (RABs). To do so, we follow the widely adopted framework of so-called artifact systems \cite{Hull08}, in their most general form structured in three components \cite{DeLV16,MSCS20}:
\begin{inparaenum}[\itshape (i)]
\item a read-only relational database (DB) with primary and foreign keys, to store background, static information;
\item a mutable working memory consisting of a set of evolving relations;
\item a set of guarded transitions that inspects the DB and the working memory and updates the latter.
\end{inparaenum}
RABs actually take the RAS model of \cite{MSCS20} (which supports the most expressive forms of guarded transitions), extending it with two essential features:
full-fledged arithmetic theories, and universal quantification in the guards and effects of transitions. 


\subsection{Static DB schemas}
\label{sec:readonly}
Static DB schemas define read-only relations with primary and foreign key constraints, which host pure identifiers subject to $\EUF$, or integer/real data attributes subject to arithmetic constraints expressed in $\LIA$/$\LRA$. 

\begin{definition} \label{def:extdb}
A \emph{static DB schema} ($\DB$ schema for short) is a pair $DB:=\tup{\Sigma^{DB}\cup\Sigma^{ar},T^{DB}\cup T^{ar}}$, where:
  \begin{compactitem}
  \item $\Sigma^{DB}$, called \emph{$\artDB$ signature},
   is a finite multi-sorted
    signature (where sorts are partitioned into \emph{id} and \emph{value} sorts $\Sorts{\Sigma^{DB}} = \ids{\Sigma^{DB}}\uplus\vals{\Sigma^{DB}}$), whose symbols are equality, unary functions, $n$-ary relations and constants;
  \item $T^{DB}$, called \emph{$\artDB$ theory}, is $\EUF(\Sigma^{DB})\cup \{  \forall x~(x = \nullv \leftrightarrow f(x) = \nullv) \}$, for every function $f$ in $\Sigma^{DB}$;
 \item $\Sigma^{ar}$, called \emph{arithmetic signature}, is the signature of $\LRA$ or of $\LIA$; 
  \item $T^{ar}$, called \emph{arithmetic theory}, is $\LRA$ or $\LIA$.
  \item $\vals{\Sigma^{DB}}\cup\Sorts{\Sigma^{ar}}$ can only be the codomain sort of a symbol from $\Sigma^{DB}$ other than an equality predicate;
  \end{compactitem}
     We respectively call $\Sigma:=\Sigma^{DB}\cup\Sigma^{ar}$ and $T:=T^{DB}\cup T^{ar}$ the \emph{full signature}  and the \emph{full theory}  of 
     $DB$.
\end{definition}
Although unconventional, the definition employs a functional approach for modeling DB schemas that can be directly processed by our technical machinery (see Section~\ref{sec:safety-uras}), 
and at the same time captures the most sophisticated read-only DB schemas considered in the literature \cite{DeLV19,verifas,BPM19}. Unary functions are used to capture relations with primary and foreign keys. Specifically, the domain of a unary function is an \emph{id sort}, representing object identifiers for that sort. Functions sharing the same id sort as domain are used to model the attributes of such objects, which can either point to other id sorts (implicitly representing foreign keys), or to so-called \emph{value sorts} that denote primitive datatypes. While in previous works only uninterpreted value sorts could be employed, here we also support real/integer datatypes, subject to arithmetic theories that considerably increases the modeling power of the language.

In the $\DB$ schema, we use default, undefined objects/values to model NULL-like constants in the working memory. To this end, we introduce a special $\nullv$ constant and explicitly define the axiom $\forall x~(x = \nullv \leftrightarrow f(x) = \nullv)$ to consistently indicate that application of a function to an undefined object returns an undefined value/object and that this is the only case for which the function is undefined.

\begin{definition} 
  \label{def:instance}
  An \emph{$\DB$ instance} of $\DB$ schema $DB:=(\Sigma,T)$ (where $\Sigma:=\Sigma^{DB}\cup\Sigma^{ar}$) is a $\Sigma$-structure
  $\cM$ that is a model of $T$ and such that every id sort of $\Sigma^{DB}$ is
  interpreted in $\cM$ on a \emph{finite} set.  
  \end{definition}
There is a key difference between $\DB$ instances and arbitrary \emph{models} of $T^{DB}\cup T^{ar}$: finiteness of id sorts 
and of the non-id values that can be pointed from id sorts using functions. This is customary for relational DBs \cite{FoundationsOfDB}.  
%
As shown in \cite{MSCS20}, $T^{DB}$ has the finite model property for constraint satisfiability, hence the standard SMT problem 
can be equivalently reformulated by asking for the existence of an $\DB$ instance instead of a generic model of $T^{DB}$.



\begin{example}
  \label{ex:db-v}
  Consider a visa application center, with a read-only DB that stores information 
  critical to the visa application process, including personal data of citizens and visa types. 
  We formalize this in a DB signature $\vsig$ with:
  \begin{inparaenum}[\itshape (i)]
  \item one id sort to identify citizens;
  \item two value sorts STRING and INT, used, e.g., for country visa types and for giving scores to visa applications under review.
  \end{inparaenum}
  %
\end{example}



\subsection{\uras Transitions}
\label{sec:transitions}

The working memory of an \uras consists of \textit{individual} and
\textit{function} variables. Function variables model evolving relations, in the style of \cite{DeLV16,verifas}, while individual variables play a twofold role: they are used to store global information about the control state of the process, as well as to load and manipulate (components of) tuples from the static relations.


Given an $\DB$ schema $\tup{\Sigma,T}$, a \emph{(working) memory extension} of $\Sigma$ is a
signature $\ext{\Sigma}$ obtained from $\Sigma$ by adding to it some extra sort
symbols together with corresponding equality predicates. 
These sorts (indicated with $E, E_1, E_2 \dots$) 
are called \emph{memory sorts}, 
whereas the ``old'' sorts from $\Sigma$ are called \emph{basic}  (variables of basic sorts are called `basic sort' variables). 
This is done to model the mutable working memory similarly to the static DB: 
implicit identifiers of working memory tuples form working memory relations, whereas data values of such tuples have basic sorts.


A \emph{memory schema} is a pair $(\ux,\ua)$ of individual and unary function variables. 
Variables in $\ux$ are called 
 \emph{memory variables}, and 
 the ones in $\ua$ \emph{memory
components}. 
The latter are required to have a memory sort as source sort and a basic sort as target sort.
Given an $\DB$ instance $\cM$ of $\ext{\Sigma}$, an \emph{assignment} to 
a memory schema $(\ux, \ua)$ over $\ext{\Sigma}$ is a map $\alpha$ assigning
to every 
 $x_i\in \ux$ of sort $S_i$ an element
$x^\alpha\in S_i^\cM$ and to every 
$a_j: E_j\longrightarrow U_j$ (with $a_j\in \ua$) a 
function
$a_j^\alpha: E_j^\cM\longrightarrow U_j^\cM$. The notion of assignment formally captures the current configuration of the working memory.

An assignment to $(\ux, \ua)$ can be seen as an $\DB$ instance
\emph{extending} the $\DB$ instance $\cM$. 
Assuming that $\ua$ contains
$a_{i_1}: E\longrightarrow S_1, \cdots, a_{i_n}:E\longrightarrow S_n$, 
the memory relation $E$ in the assignment $(\cM,\alpha)$ 
is the set $\{\tup{e, a_{i_1}^\alpha(e), \dots, a_{i_n}^\alpha(e)} \mid e\in E^{\cM} \}$.
Thus each tuple of $E$ is formed by an implicit unique ``identifier'' $e\in E^\cM$ (called \emph{index})
and by ``data'' $\ua_i^\alpha(e)$ taken from the
static DB $\cM$.  
When the system evolves, the set $E^\cM$ remains fixed, 
whereas the components $\ua_i^\alpha(e)$ may get repeatedly updated. 
``Removing'' a tuple from $E$ results in $e$ being reset to $\nullv$. This clarifies the relational nature of the working memory.

Given a memory schema $(\ux,\ua)$ over $\ext{\Sigma}$, where $\ux=x_1,\dots, x_n$ and $\ua=a_1,\dots,a_m$, 
we next list the kind of formulae that can be used in \uras{s}. 

The first kind is an \emph{initial formula}, of the form
   $\iota(\ux,\ua):= (\bigwedge_{i=1}^n x_i= c_i) \land
    (\bigwedge_{j=1}^m a_j =\lambda y. d_j)$,
where $c_i$, $d_j$ are constants from $\Sigma$ (typically, $c_i$ and $d_j$ are initially set to \nullv). 

The second kind is a \emph{state formula}, of the form $\exists \ue\, \phi(\ue, \ux,\ua)$, where $\phi$ is quantifier-free and $\ue$ are  individual variables of artifact sorts (also called \emph{`index' variables}).

The third kind is a \emph{transition formula}, of the form
   $$\univtrans(\ue,  \ud, \ux,\ua):=
\exists \ue,\ud \,\left(
      \begin{array}{@{}l@{}l@{}}
      \gamma(\ue,\ud,\ux,\ua) {}\land (\forall k \; \gamma_u(k, \ue, \ud, \ux, \ua)) \\
       {}\land\bigwedge_i x'_i= F_i(\ue, \ud, \ux,\ua) \\
       {}\land \bigwedge_j a'_j=\lambda y. G_j(y,\ue,  \ud, \ux,\ua)
      \end{array}
    \right)$$ 
  \noindent where the $\ue$ and $\ud$ are `index' and `basic sort' individual variables resp., $k$ is an individual variable of artifact sort, $\gamma$ (the ``(plain) guard'') and $\gamma_u$ (the ``universal guard'') are quantifier-free, $\ux'$ and $\ua'$ are renamed
  copies of $\ux$ and $\ua$, and the $F_i$, $G_j$ (the ``updates'') are case-defined
  functions. 
  The existentially quantified ``data'' variables $\ud$ (i.e., of basic sort) are essential as they allow to express existential queries over the $\DB$ schema, to retrieve data elements from it and also to represent  (non-deterministic) external user inputs.
  
Transition formulae as above can model such operations over tuples as
\begin{inparaenum}[\itshape (i)]
\item insertion (with/without duplicates) in a memory relation,
\item removal from a memory relation,
\item transfer from a memory relation to memory variables (and
  vice-versa), and
\item bulk removal/update of a memory relation, based on a condition expressed on such relation.
\end{inparaenum}
These operations can all be formalized as \uras{s} transitions:  modeling patterns using this approach have been shown in \cite{BPM19}.

We are now ready to formally defined \uras{s}.



\begin{definition}\label{def:uras}
  A \emph{relational action base (\uras)} is a tuple
  \[
    \cS ~=~\tup{\tup{\Sigma,T},\ext{\Sigma}, \ux, \ua, \iota(\ux,\ua),
     \tau(\ux,\ua,\ux',\ua')}
  \]
  where:
  \begin{inparaenum}[\it (i)]
  \item $DB$ is an $\DB$ schema,
  \item $\ext{\Sigma}$ is a memory extension of $\Sigma$,
  \item $(\ux, \ua)$ is a memory schema over $\ext{\Sigma}$,
  \item $\iota$ is an initial formula, and
  \item $\tau$ is a disjunction of transition formulae $\univtrans$.
  \end{inparaenum}
\end{definition}
Since $\tau$ is a disjunction of transition formulae, it symbolically represents the union of all system transitions. 
Such transitions are used to establish interaction between the static DB (with $\DB$ schema $DB$) and the working memory (with memory schema $(\ux, \ua)$).



\begin{example}\label{ex:v-short}
Let us now present a \uras $\vras$ capturing a visa application process used by the visa application center. 
Every second week all the applications get evaluated and applicants get informed about the visa decisions.
 
 $\vras$ works over the $\DB$ schema discussed in Example~\ref{ex:db-v}. 
 The working memory of $\vras$ consists of:    
 \begin{inparaenum}[\itshape (i)]
	\item a variable $\pstatev$ captures the main phases of the process;
 	\item a variable $\visastat$ stores the visa status;
 	\item a variable $\tonotify$ stores a citizen to be notified about the approved application;
 	\item a multi-instance artifact for managing visa applications.
 \end{inparaenum}
 The latter is formalized by enriching DB signature $\vsig$ with a memory sort $\appidx$ (for ``internally'' identifying the applications), and by adding a memory schema containing the following information: 
  the applicant's national identification number $\appuser :\appidx  \longrightarrow \nin$,  visa type $\appvisa  : \appidx  \longrightarrow \stringval$, evaluation score $\appscore  : \appidx  \longrightarrow \intval$, and application results $\appres:\appidx\longrightarrow \stringval$. 
 
We now showcase a few transitions for inserting and evaluating visa applications. 
 We assume that if a memory variable/component is not mentioned in a transition formula, then its values remain unchanged. 
   To insert an application into the system, the application process has to be enabled.  
  The corresponding update simultaneously 
  \begin{inparaenum}[\itshape (i)]
  \item selects the applicant's identification number and visa type and inserts them into the memory components $\appuser$ and $\appvisa$,
  \item evaluates the visa application and inserts a non-negative score into the memory component $\appscore$.
   \end{inparaenum}
   Since memory tuples must have implicit identifiers, the above insertion requires a``free'' index (i.e., an index pointing to an undefined applicant) to be selected. This is formalized as follows:
$$
    \footnotesize
    \begin{array}{@{}l@{}}
        \exists \typedvar{i}{\appidx},\exists \typedvar{a}{\nin}, \typedvar{v}{\stringval},\typedvar{s}{\intval}\\
      \left(
        \begin{array}{@{}l@{}}
          \pstatev = \enab 
          \land a \neq \nullv   \land  v\neq \nullv \land s \geq \constant{0} \\
          {}\land
         \pstatev' = \enab   \\
          {}\land  \appuser' =
                  \smtlambda{j}{
                    \smtifinline{j=i}
                      {a}
                      {\appuser[j]}
                  }
                  \\{}\land
      \appvisa' =
                  \smtlambda{j}{
                    \smtifinline{j=i}
                      {v}
                      {\appscore[j]}
                  }
                  \\{}\land
      \appscore' =
                  \smtlambda{j}{
                    \smtifinline{j=i}
                      {s}
                      {\appscore[j]}
                  }
        \end{array}
      \right)
    \end{array}
 $$
Every two weeks, submitted applications get evaluated: this is modeled by nondeterministically assigning $\constant{evaluation}$ to the variable $\pstatev$. Then, the evaluation phase concludes with highly evaluated applications being approved and other being rejected. This is realised using the following  \emph{bulk update} transition:
$$
    \small
    \begin{array}{@{}l@{}}
      \pstatev = \constant{evaluation} 
       \land \pstatev'=\evalu  \\
      \land \appres' = \smtlambda{j}{
                           \smtif{c}{\appscore[j] > \constant{80}}
                           {\constant{approved}}{\constant{rejected}}
                         }
    \end{array}
$$
  Then, if there is at least one approved application, one can nondeterministically select an applicant with the positive result to be notified using the memory variable $\tonotify$:
$$
    \footnotesize
    \begin{array}{@{}l@{}}
        \exists \typedvar{i}{\appidx}\\
      \left(
        \begin{array}{@{}l@{}}
          \pstatev = \evalu \land  \pstatev' = \notified  \\
          {}\land  \appuser[i] \neq \nullv \land \appres[i]=\constant{approved} \\
        {}\land  \tonotify'= \appuser[i]  \land \visastat'=\appres[i] \\
        \end{array}
      \right)
    \end{array}
$$
 
Finally, we demonstrate a transition with a \emph{universal guard} that checks whether no application has been approved and, if so,  changes the process state to $\constant{no-visa}$:
 \resizebox{0.48\textwidth}{!}{$
  \begin{array}{@{}l@{}}
    \pstatev = \evalu \land \forall k \; (\appres[k] \neq \constant{approved}) \\
     {}\land \pstatev'=\constant{no-visa} 
     \end{array}
$}

\end{example}

\newcommand{\breach}{\ensuremath{\mathsf{BReach}_{\text{\uras}}}\xspace}

\section{Parameterized Safety Verification}\label{sec:safety-uras}
We now turn to safety verification of an \uras $\cS$. To do so robustly, we follow the approach from the literature on artifact systems \cite{DHLV18}, and in particular the formulation given in \cite{MSCS20} for \ras{s}. As we will see next, while the problem is formulated analogously for \ras{s} and  \uras{s}, the corresponding algorithmic techniques are substantially different, and so is proving their (meta-)properties. 
The main idea is to analyse whether $\cS$ is safe \emph{independently from} the specific configuration of read-only, static data, i.e., for every instance of the $\DB$ schema of $\cS$. ``Being safe'' means that $\cS$ never reaches a state satisfying an undesired, state formula $\upsilon(\ux,\ua)$, called \emph{unsafe formula}. Technically, we say that $\cS$ is \emph{safe w.r.t.} $\upsilon$ if there is no DB-instance $\cM$ of $\tup{\ext{\Sigma},T}$, no $k\geq 0$, and
no assignment in $\cM$ to the variables $\ux^0,\ua^0 \dots, \ux^k, \ua^k$ such
that the formula
\begin{equation}\label{eq:smc1}
  \begin{array}{@{}l@{}l@{}}
    \iota(\ux^0, \ua^0)
    & {}\land \tau(\ux^0,\ua^0, \ux^1, \ua^1)\\
    & {}\land \cdots
    \land\tau(\ux^{k-1},\ua^{k-1}, \ux^k,\ua^{k})
    \land \upsilon(\ux^k,\ua^{k})
  \end{array}
\end{equation}
is true in $\cM$ ($\ux^i$, $\ua^i$ are renamed copies of $\ux$, $\ua$). Formula 
can be seen as a \emph{symbolic unsafe trace}. The safety problem $(\cS,\upsilon)$ consists of establishing whether $\cS$ is safe w.r.t $\upsilon$.

\begin{example}\label{ex:v-prop}
The following formula describes an unsafety property for the \uras from Example~\ref{ex:v-short}, checking whether the evaluation notification is directed to an applicant whose application was rejected:
$$
    \begin{array}{@{}l@{}}
      \exists  \typedvar{i}{\appidx}\\
      \left(
        \begin{array}{@{}l@{}}
        \appuser[i]\neq \nullv \land \tonotify = \appuser[i] \\
        \land \visastat = \constant{rejected} \land \pstatev = \notified
       \end{array}
     \right)
    \end{array}
  $$
%


\end{example}

\SetKwInOut{Input}{input}
\begin{algorithm}[t]
\SetKwProg{Fn}{Function}{}{end}
\Input{\small \uras $\tup{\tup{\Sigma,T},\ext{\Sigma}, \ux, \ua, \iota(\ux,\ua),
     \tau(\ux,\ua,\ux',\ua')}$}
\Input{\small (Unsafe) state formula $\upsilon(\ux,\ua)$}
\setcounter{AlgoLine}{0}
\ShowLn$P\longleftarrow \upsilon$;  $\ApproxB\longleftarrow \bot$\;
\ShowLn\While{$P\land \neg \ApproxB$ is $T$-satisfiable}{
\ShowLn\If{$\iota\land P$ is $T$-satisfiable}
{\textbf{return}  $(\unsafe, \mbox{\emph{unsafe trace} witness})$}
\setcounter{AlgoLine}{3}
\ShowLn$\ApproxB\longleftarrow P\vee \ApproxB$\;
\ShowLn$P\longleftarrow \mathit{\InstPre}(\tau, P)$\;
\ShowLn$P\longleftarrow \mathsf{Covers}(T,P)$\;
}
\textbf{return} (\safe, $\ApproxB$);}{
\caption{\breach}\label{alg1}
\end{algorithm}

\smallskip
\noindent
\textbf{Safety verification procedure.} 
Algorithm~\ref{alg1} introduces the \breach procedure for safety verification and shows how the backward reachability procedure for SMT-based safety verification, originally introduced in \cite{lmcs} for array-based systems and then refined in \cite{MSCS20} to deal with \ras{s}, can be effectively extended to handle the advanced features of \uras{s}. 
\breach takes as input a \uras $\cS$ and an unsafe formula $\upsilon$. The main loop starts from the undesired states of the system (symbolically represented by  $\upsilon$), and explores \emph{backwards} the system state space by iteratively computing, symbolically, the set of states that can reach the undesired ones (and are hence undesired as well). Every iteration of the loop \emph{regresses} the current undesired states by considering the transitions $\tau$ of $\cS$ in a reversed fashion. As we describe next, due to the 
presence of 
universal guards, 
 transitions can be reversed only in an approximated way. The computation of symbolic, approximated preimages of the current undesired state is handled in Lines~5 and~6 of \breach.

Let $\psi(\ux,\ua)$ be a state formula, describing the state of  variables $\ux,\ua$.
The \emph{exact preimage} of the set of states described by
$\psi(\ux,\ua)$ is the set of states described by
$\textit{Pre}(\tau,\psi)$ (notice that, when $\tau=\bigvee_i\univtrans_i$,
        then $\textit{Pre}(\tau,\psi)=\bigvee_i\textit{Pre}(\univtrans_i,\psi)$). This is the exact set of states that, by executing $\tau$ one time, reaches the set of states described by $\psi$.
The main issue we incur in doing so is that a state formula is an existentially quantified $\Sigma$-formula over indexes only. While, by assumption, $\psi$ is a state formula, $\textit{Pre}$ is not, making it impossible to reiterate the preimage computation. This is due to universally quantified `index' variable $k$ and existentially quantified `data' variables $\ud$ in $\tau$. To attack this problem, we introduce in Lines~5 and~6 of \breach two \emph{over-approximations}, guaranteeing that the preimage is a proper state formula and that we can \emph{regress} undesired states via preimage computation an arbitrary amount of times.

The first approximation (Line~5) compiles away the universal quantifier ranging over the index $k$ through \emph{instantiation}, by invoking $\InstPre(\tau,\psi)$. Given $\tau:=\bigvee^p_{r=1} \univtrans_r$, let $\forall k \; \gamma^{r}_u(k, \ue, \ud, \ux, \ua)$ be the universal guard of $\univtrans_r$ for all $r=1,\dots,p$. $\InstPre(\tau,\psi)$ approximates $\mathit{Pre}$ by instantiating the universally quantified `index' variable $k$ with the existential `index' variables appearing in $Pre(\univtrans_r,\psi)$, for all $r=1,\dots,p$. Formally,   given $\psi:=\exists \ue_1 \varphi_1(\ue_1,\ux,\ua)$,    $ \InstPre(\univtrans_r,\psi)$ is the formula obtained from
%
%
\begin{scriptsize}
 $
\!\exists\, \ux',\ua',\ue,\ue_1,\ud\! \left( \!\left(\!\! 
      \begin{array}{@{}l@{}l@{}}
      \gamma(\ue,\ud,\ux,\ua)\\  {}\land  \bigwedge_{k\in\ue\cup\ue_1}\gamma^{r}_u(k, \ue, \ud, \ux, \ua) \\
      {}\land\bigwedge_i x'_i= F_i(\ue, \ud, \ux,\ua) \\
       {}\land \bigwedge_j a'_j=\lambda y. G_j(y,\ue,  \ud, \ux,\ua)
      \end{array}
    \!\right)\! \land \varphi_1(\ue_1,\ux',\ua') \!\!\right)$%
    \end{scriptsize}
        by making the appropriate substitutions (followed by beta-reduction)  in order to  eliminate the existentially quantified variables $\ux'$ and $\ua'$ (see, e.g., \cite{MSCS20} for details).
The second approximation (Line~6) takes the so-computed result  $\InstPre(\tau,\psi)\equiv\exists \ue_2,\ud\, \varphi_2$, and compiles away the existentially quantified data variables $\ud$. 
 This is done by invoking $\mathsf{Covers}(T,\exists \ue_2,\ud\, \varphi_2)$, which `eliminates' the $\ud$ variables via the cover computation algorithm of $T$ described in \cite{IJCAR20}.\footnote{Recall that the combined theory $T^{DB}\cup T^{ar}$ admits covers.} Both approximation operators are  $T$-implied by the exact preimage operator. 

Together, Lines~5 and~6 produce proper state formula $\psi'$ that can be now fed into another approximate preimage computation step. In fact, \breach, iteratively computes such preimages starting from the unsafe formula $\upsilon$, until one of two possible termination conditions holds: these conditions are tested by SMT-solvers. The first condition occurs when the \emph{non-inclusion} test in Line~2 fails, detecting a \emph{fixpoint}: the set of current unsafe
 states is included in the set of states
reached so far by the search. In this case, \breach stops returning that $\cS$ is \emph{safe} w.r.t.\ $\upsilon$. The second termination condition, occurring when the \textit{non-disjointness} test in Line~3 succeeds, detects that set of current unsafe states intersects the initial states (i.e., satisfies $\iota$). In this case, \breach stops returning that $\cS$ is \emph{unsafe} w.r.t.\ $\upsilon$, together a symbolic unsafe trace of the
form~\eqref{eq:smc1}.  
 Such a trace provides a sequence of transitions
$\univtrans_i$ that, starting from the initial configurations, witness how $\cS$ can evolve from an initial state in $\iota$, under some instance of its $\DB$, to a state satisfying the unsafe formula $\upsilon$. 

Being preimages computed in an over-approximated way, unsafety may be spuriously returned, together with an unsafe spurious trace that cannot be actually produced by the \uras under study. We next study this and other meta-properties of \breach.

\smallskip
\noindent
\textbf{Meta-properties} Consider some procedure for verifying safety of \uras{s}. Given a \uras $\cS$ and an unsafe formula $\upsilon$, we say that a \safe (resp.~\unsafe) produced result is \emph{correct} iff $\cS$ is safe (resp.~unsafe) w.r.t.~$\upsilon$. We use this notion to define some key meta-properties.

\begin{definition}\label{ras-meta-prop}
Given a \uras $\cS$ and an unsafe formula $\upsilon$, a  procedure for verification unsafety of $\cS$ w.r.t. $\upsilon$ is: 
\begin{inparaenum}[\it (i)]
\item \emph{sound} if, when terminating, it returns a correct result;
\item \emph{partially sound} if a \safe result is always correct;
\item \emph{complete} if, whenever \unsafe is the correct result, then it returns so.
\end{inparaenum}
\label{def:properties}
\end{definition}


By using that reachable states via approximated preimages include those obtained via exact preimages, we get:

\begin{proposition}\label{prop:safe}
If \breach returns \safe when applied to the safety problem $(\cS,\upsilon)$, then $\cS$ is safe w.r.t. $\upsilon$.
\end{proposition}

It can be also proved (see the appendix) that in case the \uras $S$ is \emph{unsafe} w.r.t.~$\upsilon$, the procedure indeed terminates returning \unsafe. 
We next show that \breach can only be used partially to verify unsafety of \uras{s}. In the theorem, 
\emph{effectiveness} means that all subprocedures in \breach can be effectively computed (from \cite{MSCS20,arxivbeth}, the occurring $T$-satisfiability tests are decidable).

\begin{theorem}\label{thm:partial-sound}
\breach is  \emph{effective}, \emph{partially sound} and \emph{complete} when verifying unsafety of \uras{s}.
\end{theorem}

\uras{s} that do not employ universal guards collapse to \ras{s} equipped
 with numerical values and arithmetics. By carefully combining previously known results on soundness and completeness of backward reachability of \ras{s} \cite{MSCS20}, and on the existence of combined covers for \EUF\ and \LIA/\LRA \cite{JAR22}, for this class of systems we obtain the following.

\begin{theorem}
  \label{thm:semi-decision}
\breach is \emph{effective}, \emph{sound} and \emph{complete} (and hence a semi-decision procedure) when verifying safety of \uras{s} without universal guards.
\end{theorem}
This witnesses that, for this class of systems, no spurious results are produced regarding unsafety.


\section{Safety Universal Invariants}\label{sec:inva}

Injecting and exploiting invariants in the verification of dynamic systems is a well-known 
approach. Also, instantiation is a widely studied in invariant checking~\cite{shahom}. In our setting, 
invariants can dramatically prune the search space of backward reachability procedures, in general also increasing the chances that the procedure terminates~\cite{lmcs}.  
We contribute to this line by introducing a suitable notion of (universal) invariant for \uras{s}, and by studying its impact to safety verification.

\begin{definition}\label{def:invariants}
 Formula $\phi(\ux,\ua)=:\forall \ui\, \psi(\ui, \ux, \ua)$ (with $\ui$ variables of artifact sort) is an
\emph{(inductive) universal invariant} for a \uras
$\cSi$ iff (a) $T\models\iota(\ux, \ua) \Rightarrow 
\phi(\ux, \ua)$, and (b) $T\models  \phi(\ux, \ua)\wedge
\tau(\ux,\ua,\ux', \ua') \Rightarrow  \phi( \ux',\ua')$. 
 If, in addition to (a) and (b), we also have that (c)
$ \phi( \ux,\ua)\wedge \upsilon(\ux,\ua)$ is $T$-unsatisfiable,
then we say that  $ \phi(
\ux,\ua)$ is a \emph{safety universal invariant} for the safety problem
$(\cSi,\upsilon)$.  
\end{definition}
Following the arguments from~\cite{MSCS20} the $T$-satisfiability tests required in (a), (b), and (c) are
decidable. 
Using invariants for safety verification 
is, to some extent, ``more general'' than using
(variants of) backward reachability.  Unfortunately, the so-called \emph{invariant method} is more
challenging because finding safety invariants cannot be
mechanized (see, e.g., \cite{lmcs}).
Yet, if a (not necessarily safety) universal invariant 
has been found in some way (e.g.,
by using some heuristics) or supplied by a trusted knowledge source,
then it can be effectively employed in the fixpoint test of backward reachability, e.g., by  replacing Line 2 of
Algorithm~\ref{alg1} 
with:
\begin{quote}
  \centering
  2$'$ \textbf{while} ($P\wedge \mathit{Inv}\wedge \neg \ApproxB$ is $T$-sat.) \textbf{do}
\end{quote}
where $Inv$ is a conjunction of universal invariants. 
This further constrains the formula $P\wedge \neg \ApproxB$ with $\mathit{Inv}$, possibly increasing the chances to detect unsatisfiability. 


Let $P_n$ the value of the variable $P$ at the $n$-th iteration of the main loop \breach. We prove the following, key result on how safety universal invariants relate to the (approximated) preimages computed by \breach.

\begin{theorem}\label{thm:inv}
If a safety universal invariant $\phi$  exists for a \uras $\cS$ w.r.t. $\upsilon$, then for every $n \in \mathbb{N}$, we have $P_n\rightarrow \neg\phi$.
\end{theorem}

\begin{proof}[Sketch]
The proof is by induction and is fully reported in the Appendix. It is based on the argument that, in the case of RABs, it is always possible to extend a structure w.r.t. its $\Sigma$-reduct and at the same time to restrict the given structure  w.r.t. its $\ext{\Sigma}\setminus\Sigma$-reduct, and these `opposite' constructions does not interfere with each other. 
\end{proof}

By combining Definition~\ref{def:invariants} and Theorem~\ref{thm:inv} we get:
\begin{corollary}\label{cor:inv}
If there exists a safety universal invariant $\phi$ a \uras
$\cSi$, then \breach cannot return \unsafe. 
\end{corollary}

Corollary~\ref{cor:inv} implies that if there is a safety universal invariant for $\cSi$, then Algorithm~\ref{alg1} cannot find any spurious unsafe traces. We know from Proposition~\ref{prop:safe} that in case Algorithm~\ref{alg1} returns $(\safe,\ApproxB)$, 
then the system is safe: notice also that in this case, $\neg \ApproxB$  is a safety universal invariant (see the appendix). However, if it returns an unsafe outcome, the answer could be wrong due to the presence of spurious unsafe traces. The importance of Corollary~\ref{cor:inv} lies in the fact that, even if we do not know any safety universal invariant, we are assured that in case one exists, the procedure behaves well/properly/in a fair way, as no unsafe outcome can be returned and therefore no spurious trace can be detected.

\begin{example}
The 
 formula $(\visastat=\constant{rejected}\rightarrow  \tonotify=\nullv)$ is a safety universal invariant for the safety problem $(\vras,\upsilon)$.     
\end{example}


\section{Practical Verification of \uras{s}}
\label{sec:first-exp}

\begin{table}\centering
\begin{scriptsize}
\resizebox{\columnwidth}{!}{%
        \begin{tabular}{r|lrrrrrrr}\toprule
       & \textbf{Example} & \textbf{Ar} & \textbf{UGuards} &\textbf{\#AC} & \textbf{\#AV} &\textbf{\#T}& \textbf{\#Q}&\textbf{\#In}  \\\midrule
            E01& Acquisition-following-RFQ & n & y &6 &13 &28 & 14 &7\\
            E02& Airline-Check-In &  y & n &1 & 33  &  48 & 3 & 5\\
            E03& Amazon-Fulfillment & n & y & 2 & 28 & 38  &17 & 11\\
            E04& BPI-Web-Registration-with-Moderator & n &  5  & 25  & 22  & 9 & 4 \\
             E05& BPI-Web-Registration-without-Moderator & n &  5  & 25  & 20  & 9& 4 \\
             E06& Bank-Account-Opening  & y & n &  7 & 25  & 16  & 6 & 4\\
            E07& Book-Writing-and-Publishing & n & y & 4 & 14 & 14 &10 & 4\\
            E08& Commercial-Financing & n &  7 & 14  & 34  & 4 & 9 \\
             E09& Credit-Review-and-Approval & n &  12 & 24  & 23  & 3 & 13\\
             E10& Customer-Quotation-Request & n & y &  9 & 11  & 21 & 11 & 8\\
             E11& Employee-Expense-Reimbursement-Alternative-1 & y & y & 3 & 17  & 21  & 8 & 1\\
              E12& Employee-Expense-Reimbursement-Alternative-2  &y  & y & 3 & 17   & 21 & 8 & 1 \\
             E13& Incident-Management-as-Collaboration & n & y & 3 & 20  & 20  & 10 & 3 \\
              E14& Incident-Management-as-Detailed-Collaboration & n & y &   3 & 20  & 20  &10 & 3 \\
              E15& Insurance-Claim-Processing &  y & n &4 &  22 & 22 &13 & 3 \\
             E16& Journal-Review-Process & n &  6 & 25  & 47 & 19 & 9 \\
             E17& LaserTec-Production-Process & y & n & 0  & 18  & 13 & 8  & 0 \\
             E18& Mortgage-Approval & n &  3 & 18  & 21 & 9 & 3 \\
              E19& New-Car-Sales & y & n & 0 & 23  & 31 & 10 & 0\\
             E20& Order-Fulfillment-and-Procurement & n &  3 & 11  & 24 & 7 & 2 \\
              E21& Order-Fulfillment & y & y & 7 & 17  & 27 & 7 & 4\\
              E22& Order-Processing-with-Credit-Card-Authorization  & y & y & 9 & 18  & 20 & 2  & 7\\
               E23& Order-Processing & y & n &  9 &  14 & 17  & 2 & 7 \\
               E24& OrderFulfillment\_new & y & y & 1 & 17   & 15 &  7 & 4\\
                E25& Patient-Treatment-Abstract-Process & n & y &  6 & 17  & 34 & 14 & 20\\
                E26& Patient-Treatment-Collaboration-Choreography & n & y & 6 & 17  & 34 & 15 & 20\\
	     E27& Patient-Treatment-Collaboration & n &  y &6 & 17  & 34 & 15 & 20 \\
             E28& Pizza-Co.-Delivery-Process & y & y & 2 & 32  & 32 & 10 & 2 \\
              E29& Property-and-Casualty-Insurance-Claim-Processing & n & y & 2  & 7  & 15 & 3 & 3\\
               E30& Ship-Process-of-a-Hardware-Retailer & y & n & 0 & 28  & 26 & 9 & 0 \\
                E31& The-Pizza-Collaboration & y & y & 2 & 32  & 37 & 12 & 2 \\
                E32& Travel-Booking-with-Event-Sub-processes & y & n & 14 & 32  & 51 & 9  & 14\\
                E33& Travel-Booking & y &  n &14  & 32 &  43 &  8 & 10 \\
            E+& JobHiring & y & y &9 & 18 & 15 & 7 & 6\\
            \\\bottomrule
        \end{tabular}
        }
        \caption{Summary of the tested examples}\label{tab:benchmark}
        \end{scriptsize}
    \end{table}

We now put our verification machinery in practice, exploiting the fact that, as shown in our technical development, we can employ the different modules of \mcmt to handle parameterized safety checking of \uras{s}. Specifically, we conduct an extensive experimental evaluation on \uras{s} representing concrete data-aware business processes. Our evaluation builds on the existing benchmark provided in \cite{verifas}, which samples 33 (32 plus one variant) real-world BPMN processes published in the official BPM website (\url{http://www.bpmn.org/}).\footnote{BPMN is the de-facto standard to capture business processes.}
We encode \emph{all} these 33 models
into the array-based specifications of \uras{s}, by exploiting the syntax of  the database-driven mode of \textsc{mcmt} (from Version~2.8). To set up our benchmark, we enrich 26 BPMN models out of the 33  that lend themselves to be directly extended with arithmetic and/or universal guards (features not supported by VERIFAS \cite{verifas}), and we construct 26 corresponding RABs that incorporate such advanced features. The remaining examples are reported for the sake of completeness, but fall into the already studied class of  RASs \cite{MSCS20}. We also develop a further example that incorporates all the modelling features of RABs (for example, \emph{bulk updates} that go beyond the capabilities of VERIFAS \cite{verifas}), with the intention of stress-testing MCMT. 
Each example from the benchmark is checked against 12 conditions, where at least one is safe and one is unsafe; our running example is checked against 33 conditions. Overall, we ran \textsc{mcmt} over 429 specification files.

Experiments were performed on a machine with macOS High Sierra 10.13.3, 2.3 GHz Intel Core i5 and 8 GB RAM.  The full benchmark set is available on GitHub \cite{benchmark-rabs}. 
%
%


 \begin{table} \centering
 \begin{tiny}
        \begin{tabular}{rrrrrrrrr}\toprule
            \textbf{Ex} & \textbf{\#U}&\textbf{\#S}&\textbf{MeanT} &\textbf{MaxT} &\textbf{StDvT} & \textbf{Avg\#(N)}  & \textbf{AvgD} &\textbf{Avg\#(calls)}
            \\\midrule
            E01& 9 & 3 & 0.82 & 1.28 (U) &  0.36 & 60.1 & 8.25  & 3398.2 \\
             E02 & 6 & 6 & 0.31 & 0.37 (S) & 0.04 & 6.7 & 4.33 & 4011.4\\
             E03 & 6 & 6 & 0.74 & 2.49 (S) & 0.69 & 24.4 & 6.75 & 4301.8\\
             E04 & 5 & 7 & 0.32 & 0.64 (U) & 0.18 & 15.0 & 7.67 &1899.7\\
             E05 & 5 & 7 & 0.23 & 0.51 (U) & 0.14 & 10.9 & 6.17 &1656.8\\
             E06 & 5 & 7 & 0.17 & 0.34 (U) & 0.08 & 14.6 & 6.58 &1496.4\\
             E07 & 10 & 2 & 0.18 & 0.62 (U) & 0.16 & 21.8  & 3.75 &1126.1 \\
             E08 & 6 & 6 & 0.65 & 2.50 (S) & 0.67 & 34.2  & 8.75 & 3036.4 \\
             E09 & 9 & 3 & 19.76 & 171.63 (S) & 48.42 & 175.8 & 12.00 & 20302.4 \\
	     E10 &  9  &  3  &   0.22      &   0.46  (S)  &  0.14  &  15.1   &  4.75   & 1550.3 \\
	     E11 &   7 &  5  &    0.15     &   0.38 (S)   &  0.10  &  17.3   &  7.25   & 1339.6 \\
	     E12 &  7  &  5  &   0.15      &   0.38 (S)    & 0.10   &  16.4   &  6.92   & 1326.7 \\
	      E13 &  8  &  4  &    0.65     &  2.39 (U) &  0.80  &  47.4   &  7.67   & 2764.8  \\
	     E14 &  8  &  4  &   0.62      &   2.24  (U)    &  0.77  &   46.5  & 7.67    & 2724.0 \\
	     E15 &  7  &  5  &   0.37      &   0.76 (U)    &  0.21  &  32.3   &  9.5   & 2283.4 \\
	      E16 &  8  & 4   &   0.96      &    5.47  (U)   &  1.43  &  38.3   & 11.0    & 5151.8 \\
	       E17 &  7  & 5   &   0.08       &    0.13 (U)   &  0.02  &  11.3   &  8.67   & 768.5 \\
	        E18 &  6   &  6  &   0.11      &    0.19 (U)  &  0.04  & 9.0    &  5.33   & 1212.6 \\
	        E19 &  6  &  6  &   0.37      &   0.74 (U)    &   0.15 &  31.3     &   8.17  & 2416.2 \\
	        E20 &  6  &  6  &   0.11      &   0.17 (U)   & 0.03   &  10.9  &  5.5   & 1026.5 \\
	        E21 &  7  &  5   &   0.40      &    1.38 (S)   &  0.33  &  30.8   &  13.25   & 2509.6 \\
	        E22 &  9  &  3  &   8.15      &   53.92 (U)    & 17.11  &  134.3   & 8.17    & 9736.3 \\
	        E23 &  8  & 4   &    0.99     &  3.73 (S)     &  1.48  & 35.3    &   7.09  & 2995.9 \\
	        E24 &   7 &  5  &   0.11      &  0.13 (U)    &  0.02  &   9.5  &  5.42   & 1051.5 \\
	        E25 &  6  &  6  &    4.52     &   24.74 (S)    &  9.43  &  27.1   &  4.83   & 7486.7 \\
	        E26 &  6  &  6  &     4.97    &    17.79 (S)   &  7.55  &  28.2   &   4.92  & 4811.1  \\
	        E27 &  7  & 5   &   4.59      &   20.81 (S)    & 7.42   &  26.9   &   4.91  & 4633.7 \\
	        E28 &  7  &    5  &     0.23    &  0.46 (U)    & 0.08  &  9.1   &   6.17  &  2740.0  \\
	        E29 & 6   &  6  &   0.08      & 0.42 (S)    & 0.11   & 11.3    &   5.42  &  648.8 \\
	        E30 &  6  &   6 &   0.31      &   0.79 (U)    &  0.20  &  26.4   &  5.42   & 2316.2 \\
	        E31 &  8  & 4   &    0.38    &    0.71 (U)   &  0.18  & 17.9    &   8.58  & 3367.0 \\
	        E32 &  9  &  3  &   2.48      &   8.49 (U)    &  2.41  & 97.3    &  17.75   & 9231.3 \\
	        E33 &   9 &  3  &   1.27      &   4.24  (S)    &   1.16 &   66.7  &   16.83  & 6637.8 \\
	        E+ & 22   &  11  &    7.39     &   98.27 (S)    & 23.55   &  75.7   &  9.15  & 5612.5 \\
             %
            \\\bottomrule
        \end{tabular}
        \caption{Experimental results for safety properties}\label{tab:exp_extended}
        \end{tiny}
\end{table}

We provide two tables describing the benchmark run in \mcmt: Table~\ref{tab:benchmark} lists the name of tested examples and gives relevant information on the size of the input specification files, whereas Table~\ref{tab:exp_extended} summarizes the experimental results obtained running \mcmt over those files. For each example, Table~\ref{tab:benchmark} reports:
\begin{inparaenum}[\it (i)]
\item the number of memory components (\textbf{\#AC}); 
\item the number of memory variables (\textbf{\#PV}); 
\item the number of transitions (\textbf{\#T}), also counting those that contain quantified data variables (\textbf{\#Q}), and those that manipulate at least one index (\textbf{\#In});
\item whether the example contains arithmetics (\textbf{Ar)}) and universal guards (\textbf{UGuards)}).
\end{inparaenum}

Table~\ref{tab:exp_extended} shows several measures\footnote{A measure is ``average for an example'' when it is the mean of that measure for all properties tested on that example.}: 
\begin{inparaenum}[\it (i)]
\item number of \unsafe (column \textbf{\#U}) and \safe (\textbf{\#S}) outcomes;
\item seconds for the average (\textbf{MeanT}) and maximum (\textbf{MaxT}) \mcmt execution time, together with the standard deviation (\textbf{StDvT});
\item indications of the size of the backward state space explored by \mcmt, in terms of average number of nodes (\textbf{Avg\#(N)}), average depth of the search tree (\textbf{AvgD}) and of number of calls to the external SMT solver (\textbf{Avg\#(calls)}).
\end{inparaenum}

Considering only the 26 plus 1 models supporting the specific features of RABs, we notice from Table~\ref{tab:exp_extended} that the means of the execution times are relatively small: as one can see from the  {\tt RAB-benchmark-time-log.txt} file in URL, \mcmt terminates in less than one second for $85.2\%$ of the tested files (294 out of 345), and in less than one minute for 343 out of 345. The two files over one minute take $90.36$ and $98.27$ seconds, respectively. 
Considering all the 33  plus 1 models, \mcmt terminates in less than one second for $85.3\%$ of the tested files (366 out of 429): the highest timing is 171.63 seconds.

When the outcome is unsafe, \mcmt returns a witness from which it is possible to symbolically reconstruct the unsafe trace as a formula. This formula falls into a decidable fragment (the exist-forall one), where we can concretely test whether the unsafe trace is spurious or not. This method can be, in principle, fully automatized (even if, currently, \mcmt does not support it).

\section{Conclusions}
\label{sec:conclusions}
We have presented the \uras framework for modeling rich relational dynamic systems equipped with arithmetics and universal guards, developing techniques for safety verification and invariant-based analysis. Our results do not only have a foundational importance, but also witness that separately implemented verification modules in the MCMT model checker, respectively dealing with static relational data, and with approximation for arithmetics and universally quantified variables, can be indeed gracefully combined to obtain a safety verification procedure for \uras{s}. The experimental evaluation here reported also shows that this is an effective, practical approach.
%
In the paper, we have provided a list of concrete models of relational dynamic systems that require the advanced features of \uras{s}, in the context of multiagent systems \cite{FeGM21} and extensions of Petri nets \cite{Fahland19,PWOB19,BPM20}. The next step is then to fully encode such concrete models into \uras{s}, also studying which relevant universal invariants can be established as a by-product of the encoding itself.

%
%


\clearpage

\bibliographystyle{named}
\bibliography{ijcai22}

\clearpage
\appendix


\section{Appendix}
In this appendix, we provide full proofs of all the results of the paper. 

\subsection{Results from Section~\ref{sec:safety-uras}}
It can be easily seen that Algorithm~\ref{alg1} returns a \safe outcome iff a fixpoint is reached, i.e., for some natural $M$, $\T\models \ApproxB_M \leftrightarrow \ApproxB_{M+1}$ and $\iota\land \ApproxB_{M}$ is $\T$-unsatisfiable.

From the previous lemma, it follows the following proposition.

\begin{lemma}\label{safety-preserv}
\begin{inparaenum}[(i)]
\item For all $k>0$, $\T\models B_k\rightarrow \ApproxB_k$. 
\item Moreover, let us suppose that for some $M$ 
$\T\models \ApproxB_M \leftrightarrow \ApproxB_{M+1}$. Then, for all $k>M$, $\T\models B_k\rightarrow \ApproxB_M$.
\end{inparaenum}

\end{lemma}


\begin{proof}
We first prove the following \textbf{fact}: for all $k>0$, $\T\models\ApproxB_k \lor \mathit{Pre}(\tau,\ApproxB_{k})\rightarrow \ApproxB_{k+1}$. Indeed, by definition of the operators $\InstPre$ and $\mathsf{Covers_{\text{\uras}}}$, which are over-approximations of $\mathit{Pre}$,  we get that $\T\models\mathit{Pre}(\tau,\psi)\rightarrow \ApproxPre(\T,\tau,\psi)$, for every formula $\psi$. Hence, we obtain that $\T\models\ApproxB_k \lor \mathit{Pre}(\tau,\ApproxB_{k})\rightarrow \ApproxB_k \lor  \ApproxPre(\T,\tau,\ApproxB_k)$, but  $\ApproxB_k \lor\ApproxPre(\T,\tau,\ApproxB_k)\equiv \ApproxB_{k+1}$, which proves the \textbf{fact}.

We now prove point (i), i.e. for all $k>0$, $\T\models B_k\rightarrow \ApproxB_k$. The base case ($k=1$) is trivial.
By inductive hypothesis, suppose that $\T\models B_k\rightarrow \ApproxB_k$.
 Since the operator $\mathit{Pre}$ is monotone, we get, by using twice the inductive hypothesis, that $\T\models B_k\lor\mathit{Pre}(\tau,B_k)\rightarrow \ApproxB_k \lor \mathit{Pre}(\tau,\ApproxB_k)$, which means, by definition of $B_{k+1}$, that $\T\models B_{k+1}\rightarrow \ApproxB_k \lor \mathit{Pre}(\tau,\ApproxB_k)$. Using the \textbf{fact}, we deduce that 
 $\T\models B_{k+1}\rightarrow \ApproxB_{k+1}$, as wanted. Point (ii) of the lemma immediately follows.
\end{proof}

\vskip 2mm\noindent
\textbf{Proposition~\ref{prop:safe}} \emph{
If Algorithm~\ref{alg1} returns \safe when applied to the safety problem $(\cS,\upsilon)$, then the \uras $\cS$ is safe w.r.t. $\upsilon$.
}
\vskip 1mm
\begin{proof}
If Algorithm~\ref{alg1} returns \safe, a fixpoint has been reached, i.e. for some natural $M$, $\T\models \ApproxB_M \leftrightarrow \ApproxB_{M+1}$ and $\iota\land \ApproxB_{M}$ is $\T$-unsatisfiable.
By Lemma~\ref{safety-preserv}, we know that for $k>M$, $\T\models B_k\rightarrow \ApproxB_M$, which implies that also $\iota\land \B_{k}$ is $\T$-unsatisfiable for every $k>M$: indeed, if $\iota\land \B_{k'}$ were $\T$-satisfiable for some $k'>M$, by $\T\models B_k'\rightarrow \ApproxB_M$ we would get that also $\iota\land \ApproxB_{M}$ is $T$-satisfiable, contradiction. Thus, by definition of the safety problem and of $\B_k$, it immediately follows that $\cS$ is safe w.r.t. to $\upsilon$.
\end{proof}

When $\cS$ is unsafe w.r.t.$\upsilon$, notice that Algorithm~\ref{alg1} cannot return a \safe outcome, otherwise $\cS$ would be unsafe w.r.t $\upsilon$. More precisely, Algorithm~\ref{alg1} is also complete: 

\begin{proposition}\label{prop:unsafe}
 If $\cS$ is unsafe w.r.t $\upsilon$, Algorithm~\ref{alg1} terminates and returns an \unsafe outcome.
\end{proposition}
\begin{proof} 
If $\cS$ is unsafe w.r.t $\upsilon$, then by definition of the safety problem, there exists a $k$ s.t.  $\iota\land B_k$ is $T$-satisfiable. By point (i) of Lemma~\ref{safety-preserv}, we know that $\T\models B_k\rightarrow \ApproxB_k$, hence $\iota\land \ApproxB_k$ is $T$-satisfiable as well. This is precisely the condition of the $\textbf{if}$ in Line~3: thus, we conclude that Algorithm~\ref{alg1}  terminates with an \unsafe outcome.
\end{proof}

Putting together the previous results, we can prove the following theorem.

\vskip 2mm\noindent
\textbf{Theorem~\ref{thm:partial-sound}}\emph{
\breach is  \emph{effective}, \emph{partially sound} and \emph{complete} when verifying unsafety of \uras{s}.
}
\vskip 1 mm
\begin{proof}
The fact that the procedure is partially sound and complete follows from Propositions~\ref{prop:safe} and \ref{prop:unsafe}. Notice also that the procedure is effective since: \begin{inparaenum} [(i)]
\item $T$-covers exist and can be effectively computed using the results from~\cite{JAR21,JAR22} 
\item all the $T$-satisfiability tests are decidable according to the results from~\cite{MSCS20}.
\end{inparaenum}
\end{proof}

If we restrict our attention on \uras{s} that do not contain universal guards, we can prove an even stronger result.

\vskip 2mm\noindent
\textbf{Theorem~\ref{thm:semi-decision}}\emph{
\breach is \emph{effective}, \emph{sound} and \emph{complete} (and hence a semi-decision procedure) when verifying safety of \uras{s} without universal guards.
}
\vskip 1 mm
\begin{proof}
From Theorem~\ref{thm:partial-sound}, the only thing that remains to prove is that Algorithm~\ref{alg1} is fully sound if the \uras does not contain universal guards. In this case, $\InstPre\equiv Pre$, hence Algorithm~\ref{alg1} coincides with the SMT-based backward reachability procedure for RASs introduced in~\cite{MSCS20}.
Notice, however, that \uras{s} without universal guards still strictly extend plain \ras{s} from \cite{MSCS20}, since $\DB$ are strictly more expressive than DB schemas of \ras{s}: $\DB$ schema contain arithmetic operations and are constrained by theory combinations where one component is an arithmetical theory. Nevertheless, using the results and the procedures presented in \cite{JAR21,arxivbeth} for computing covers, and given the equivalence between computing covers and eliminating quantifiers in model completions from~\cite{MSCS20}, in order to conclude one can use the analogous arguments based on model completions as the ones used in~\cite{MSCS20} for proving full soundness of RASs.
\end{proof}

\subsection{Results from Section~\ref{sec:inva}}
The following result is useful since it shows that Algorithm~\ref{alg1} can be used to find safety universal invariants: indeed, when the outcome is \safe, we can extract a safety universal invariant.

\begin{proposition}
If a Algorithm~\ref{alg1} returns $(\safe, \ApproxB)$, then $\neg\ApproxB$ is a safety universal invariant.
\end{proposition}

\begin{proof}
Clearly, by construction, $\ApproxB$ is an existential formula over artifact sorts only, hence $\neg\ApproxB$ is universal. We show that conditions (a),(b),(c) from Definition~\ref{def:invariants} hold.
Let us suppose that we obtained the $\safe$ outcome after $n$ iterations of the main loop: at this point, we have that $\ApproxB\equiv P_1\lor\dots\lor P_{n-1}$.
\noindent \textbf{Condition (a)}: since Algorithm~\ref{alg1} does not return \unsafe, in every iteration of the main loop the condition for the \textbf{if} construct is not satisfied: hence, for every $i=0,...,n-1$, $T\models \iota \rightarrow \neg P_i$, which means that $T\models \iota \rightarrow \bigwedge^{n-1}_{i=0}\neg P_i$, i.e., $T\models \neg \ApproxB$, as wanted.

\noindent \textbf{Condition (b)}: we need to prove that $T\models \neg\ApproxB(\ux,\ua)\land \tau(\ux,\ua,\ux',\ua')\rightarrow \neg\ApproxB(\ux',\ua')$. By trivial logical manipulations, we deduce that the thesis is equivalent to
$T\models \ApproxB(\ux,\ua)\lor \neg \exists \ux',\ua' (\tau(\ux,\ua,\ux',\ua')\land \ApproxB(\ux',\ua'))$, which means $T\models  \exists \ux',\ua' (\tau(\ux,\ua,\ux',\ua')\land \ApproxB(\ux',\ua'))\rightarrow \ApproxB(\ux,\ua)$. By definition of exact preimage, we equivalently get:
$$T\models  Pre(\tau,\ApproxB)\rightarrow \ApproxB$$
Hence, we need to prove that $T\models  Pre(\tau,\ApproxB)\rightarrow \ApproxB$. Let $\cA$ be a  $\DB$ instance, and let us suppose that $\cA\models Pre(\tau,\ApproxB)(\ux,\ua)$ for every pair $(\ux,\ua)$. 

By construction of $\ApproxB$, 
we then get that $\cA\models \bigvee^{n-1}_{i=0} Pre(\tau, P_i)(\ux,\ua)$ for every pair $(\ux,\ua)$. By definition of $\ApproxPre$ (which is an over-approximation of $Pre$), we know that $T\models Pre(\tau, P_i)(\ux,\ua)\rightarrow \ApproxPre(T,\tau,P_i)$ for all $i=0,...,n-1$, hence $\cA\models \bigvee^{n-1}_{i=0} \ApproxPre(T, \tau, P_i)(\ux,\ua)$  for every pair $(\ux,\ua)$. Thus, since $P_{i+1}:= \ApproxPre(T, \tau, P_i)$, we get:
$$\cA\models \bigvee^{n}_{i=1} P_{i}(\ux,\ua) $$
By hypothesis, we have that at the $n$-th iteration, since $\ApproxB\equiv \ApproxB_{n}$, $T\models P_{n}\rightarrow \ApproxB$, hence we conclude that
$$\cA\models  \bigvee^{n-1}_{i=0} P_{i}(\ux,\ua) \equiv \ApproxB(\ux,\ua),  $$
as wanted.



\noindent \textbf{Condition (c)}: by definition of $\ApproxB$, we clearly have that $T\models P_0 \rightarrow \neg \ApproxB$, which means $T\models \upsilon \rightarrow \neg \ApproxB$.
\end{proof}

The following theorem is instrumental to prove the main result of the section.

\vskip 2mm\noindent
\textbf{Theorem~\ref{thm:inv}} \emph{
If a safety universal invariant $\phi$  for a \uras $\cS$ w.r.t. $\upsilon$ exists, then is such that $P_n\rightarrow \neg\phi$
}
\vskip 1mm
\begin{proof}
We prove the statement by induction on the number $n$ of iterations of the main loop in Algorithm~\ref{alg1}.

\noindent\textbf{Base case.} In case $n=0$, let $P_0$ be the value of the variable $P$  in Line~1 of  Algorithm~\ref{alg1}: $P_0$ coincides with the unsafety formula $\upsilon$.  Since $\phi$ is by hypothesis an invariant, condition (b) in the definition of invariant implies that $T\models\upsilon\rightarrow \neg \phi$, which means that $T\models P_0\rightarrow \neg \phi$, as wanted.

\noindent\textbf{Inductive case.} Let $P_n$ the value of the variable $P$ at the $n$-th iteration of the main loop. By inductive hypothesis, we have that $T\models P_n\rightarrow \neg \phi$. We show that $T\models P_{n+1}\rightarrow \neg \phi$ holds as well.
First, we notice that, by condition (c) of Definition~\ref{def:invariants}, the fact that $\phi$ is invariant implies that $T\models \tau(\ux,\ua,\ux',\ua')\land \neg\phi(\ux',\ua')\rightarrow \neg \phi(\ux,\ua)$ holds. The previous implication can be rewritten as 
\begin{equation}\label{eq:inv1}
 T\models \mathit{Pre}(\tau,\neg\phi)\rightarrow \neg \phi
 \end{equation}
 since, by definition of preimage, we have $T\models \tau(\ux,\ua,\ux',\ua')\land \neg\phi(\ux',\ua')\equiv \mathit{Pre}(\tau,\neg\phi)$.

It easily follows from the definition of preimage the fact that the preimage is monotonic, i.e., given two formulae $\alpha$ and $\beta$, $ \alpha\rightarrow \beta$ implies $\mathit{Pre}(\tau,\alpha)\rightarrow\mathit{Pre}(\tau,\beta)$. Hence, from the inductive hypothesis $T\models P_n\rightarrow \neg \phi$, we get  
\begin{equation}\label{eq:inv2}
T\models  \mathit{Pre}(\tau,P_n)\rightarrow \mathit{Pre}(\tau,\neg\phi).
\end{equation}

From Formulae~\eqref{eq:inv2} and~\eqref{eq:inv1} it follows  
\begin{equation}\label{eq:inv3}
T\models \mathit{Pre}(\tau,P_n)\rightarrow \neg\phi.
\end{equation}

It can be easily seen that if $P_n$ is an existential formula, then so is also $\InstPre(\tau,P_n)$. Since $\upsilon$ is an existential formula over indexes, then we conclude by construction of $\InstPre$ and by induction that $ \InstPre(\tau,P_n)$  is an existential formula over artifact sorts (indexes) and basic sorts. Moreover, by applying the further approximation $\mathsf{Covers_{\text{\uras}}}$, we eliminate the existentially quantified variables over basic sorts, so as to get an existential formula over indexes only. Hence, since $P_{n+1}\equiv \mathsf{Covers_{\text{\uras}}}(T,\InstPre(\tau,P_n))$, let us assume that $P_{n+1}$ has the form $\exists \hat{\ue} \, \psi(\hat\ue)$, where $\psi$ is quantifier-free. Thus, in order to conclude the proof, we need to prove:
\begin{equation}\label{eq:inv4}
T\models \exists\hat{\ue} \, \psi(\hat\ue)\rightarrow \neg\phi.
\end{equation}
By reduction to absurdum, let us suppose that there exists a $DB$ instance $\cA:=(\cA_{art}, \cA_{DB})$, where $\cA_{art}$ is the $(\ext{\Sigma}\setminus \Sigma)$-reduct of $\cA$ and $\cA_{DB}$ is the $\Sigma$-reduct of $\cA$, such that $\cA\not\models \exists\hat{\ue} \, \psi(\hat\ue)\rightarrow \neg\phi$, i.e. 
$$ \cA\models  \exists\hat{\ue} \, \psi(\hat\ue)    \mbox{ and }  \cA\models \phi $$
Let $\cB_{DB}$ be a $\Sigma$-structure (which is model of $T$) that extends $\cA_{DB}$ s.t. $\cB_{DB}\models \InstPre(\tau,P_n)$: $\cB_{DB}$ exists since  $\mathsf{Covers_{\text{\uras}}}$ computes $T$-covers and thanks to the Covers-by-Extensions Lemma \cite{JAR21}. Notice that in the $\ext{\Sigma}$-structure $\B_1:=(\cA_{art},\cB_{DB})$ the formula $\phi$ still holds (its  $(\ext{\Sigma}\setminus \Sigma)$-reduct is by definition the same as the one of $\cA$).

Now, we take the
 $\ext{\Sigma}\setminus \Sigma$-substructure $\cB_{art}$ of $\cB_1$ generated by the evaluation  in $\cA_{art}$ of the variables $\hat\ue$ from $\hat{\ue} \, \psi(\hat\ue)$. Let $\cB_2:=(\cB_{art},\cB_{DB})$, which is again a model of $T$. Since validity of universal formulae is preserved when passing to substructures and $\phi$ is universal, we have that $\cB_2\models \phi$.
By construction, we still have that $\cB_2\models\InstPre(\tau,P_n)$. This implies also that $\cB_2\models \mathit{Pre}(\tau,P_n)$: indeed, all the universal guards  $\forall k \gamma_u(k, \ue,\ud, \ux,\ua)$ appearing in $\tau$ are satisfied iff the corresponding instantiated conjunctions  $\bigwedge_{k\in \ue} \gamma_u(k,\ue, \ud, \ux, \ua) $ appearing in $\InstPre(\tau,P_n)$ are satisfied, since, by construction of $\cB_2$, the instantiation of the universal quantifiers $\forall k$ in $\tau$ covers all the elements in the support of $\cB_{art}$. However,  $\cB_2\models \mathit{Pre}(\tau,P_n)$ implies that there exists a DB instance, which is $\cB_2$, such that $\cB_2\models \mathit{Pre}(\tau,P_n)\land\phi$: this contradicts entailment~\eqref{eq:inv3}, and concludes the proof.
\end{proof}

\begin{remark}
Notice  that the argument of the previous proof works because in our context it is always possible to extend a structure wrt its $\Sigma$-reduct and at the same time to restrict the given structure  wrt its $\ext{\Sigma}\setminus\Sigma$-reduct, and these `opposite' constructions does not interfere each other. Indeed, the only link between the interpretations of artifact sorts and the ones of basic sorts are the artifact components, which are free functions from `indexes' to the proper $\DB$ instance and which can always be restricted in their domain or extended in their co-domain. This would not be possible anymore in case there were function or relation symbols that can interfere each other in more sophisticated way.
\end{remark}

The following is the main result of the section.
\vskip 2mm\noindent
\textbf{Corollary~\ref{cor:inv}} \emph{
If there exists a safety universal invariant $\phi$ a \uras
$\cSi$, then Algorithm~\ref{alg1} cannot terminate with the \unsafe outcome. 
}
\vskip 1mm
\begin{proof}
By Theorem~\ref{thm:inv}, we get that, for every iteration of the main loop, $\cC(DB)\models P_n\rightarrow\neg\phi$.
Since $\phi$ is a universal invariant, from point (a) of Definition~\ref{def:invariants}, we know also that $\cC(DB)\models \neg\phi \rightarrow\neg\iota$. Hence, we deduce $\cC(DB)\models P_n\rightarrow\neg\iota$. This implies that the condition of the satisfiability test of Line~3 of Algorithm~\ref{alg1} never holds, which means that Algorithm~\ref{alg1} cannot return \unsafe.
\end{proof}


\end{document}